\documentclass[letterpaper, 10 pt, conference]{ieeeconf}  

\IEEEoverridecommandlockouts                              

\usepackage{xcolor}
\usepackage{amsmath,amssymb,amsfonts,color,nicefrac}
\usepackage{graphicx}
\usepackage{algorithm}
\usepackage{algpseudocode}
\usepackage{caption}
\usepackage{subcaption}
\usepackage{cite}

\DeclareMathOperator*{\argmin}{arg\,min}
\DeclareMathOperator{\tr}{tr}

\renewcommand{\Re}{\mathbb{R}}
\newcommand{\A}{\mathcal{A}}

\newtheorem{proposition}{Proposition}

\title{\LARGE \bf
Communication-Aware Map Compression \\for Online Path-Planning
}

\author{Evangelos Psomiadis, Dipankar Maity, Panagiotis Tsiotras
\thanks{The work was supported by the ARL grant ARL DCIST CRA W911NF-17-2-0181.}
\thanks{E. Psomiadis and P. Tsiotras are with the D. Guggenheim School of Aerospace Engineering, Georgia Institute of Technology, Atlanta, GA, 30332-0150, USA. Email:
       \{epsomiadis3, tsiotras\}@gatech.edu}%
\thanks{D. Maity is with the Department of Electrical and Computer Engineering, University of North Carolina at Charlotte, NC, 28223-0001, USA. Email:
       dmaity@charlotte.edu} %
}

\begin{document}
\maketitle
\thispagestyle{empty}
\pagestyle{empty}

\begin{abstract}
This paper addresses the problem of the communication of optimally compressed information for mobile robot path-planning.
In this context, mobile robots compress their current local maps to assist another robot in reaching a target in an unknown environment. 
We propose a framework that sequentially selects the optimal compression, guided by the robot's path, by balancing the map resolution and communication cost. 
Our approach is tractable in close-to-real scenarios and does not necessitate prior environment knowledge.
We design a novel decoder that leverages compressed information to estimate the unknown environment via convex optimization with linear constraints and an encoder that utilizes the decoder to select the optimal compression.
Numerical simulations are conducted in a large close-to-real map and a maze map and compared with two alternative approaches.
The results confirm the effectiveness of our framework in assisting the robot reach its target by reducing transmitted information, on average, by approximately 50\% while maintaining satisfactory performance.
\end{abstract}

\section{Introduction}
Advancements in the field of multi-robot decision-making enable teams of robots to carry out complex tasks such as search and rescue operations \cite{Sugiyama2005}, autonomous delivery \cite{salzman2020}, or even space missions \cite{Kesner2007}. 
These operations often center around collaborative navigation in unknown environments, where the robots engage in continuous information exchange.
However, to fully harness the potential of their communication network and optimize performance, the robots must be aware of their bandwidth limitations, and incorporate those in their control and decision-making process. 
The problem of multi-robot path-planning under bandwidth constraints is an active area of research and several approaches have been proposed, handling different aspects of the problem \cite{Gielis2022}. 

Prior robot control algorithms treated communication as an afterthought.
For example, in \cite{kepler2020}, given a data set, the authors assign metrics to assess the significance of data points and decide their communication for exploration missions. 
In a recent work for navigation/path-planning \cite{chang2023}, the authors compress a 3D Scene Graph, given the high-resolution path by keeping a specific number of nodes.
Our approach compresses the essential information for path-planning online, integrating the compressed map into the planning loop.

The coupled problem between compression (quantization) and control has been an active area of research for several decades in the controls community  \cite{delchamps1990stabilizing, brockett2000quantized,  nair2004stabilizability, kostina2019rate}. 
The choice of the optimal quantizer/compressor even for a given control objective (e.g., a quadratic cost function) is, however, an intractable problem \cite{fu2012lack, yüksel2019note}.
Instead of designing the optimal quantizer, an alternative approach is to select the best quantizer at each time from a given set of quantizers \cite{maity2021optimal, maity_quant}.
This results in a tractable linear program (LP). 
We adopt this approach, where a set of compressors/quantizers is available to the robots to compress their perception data before transmitting it to another robot in their team.

\textbf{Related Work:}
Our primary focus is determining what information to communicate in the context of a multi-agent navigation problem involving agents with different objectives. 
In \cite{unhelkar2016}, the authors introduce ConTaCT, a policy that addresses when to communicate information in multi-agent navigation scenarios by solving a decentralized Markov Decision Process with the team reward dependent on the joint action space. 
In \cite{Wu2011}, the agents communicate whenever there is an inconsistency in their shared belief.
In \cite{marcotte2020}, the authors address the problem of deciding what information to communicate using OCBC, an algorithm that employs forward simulations and a bandit-based combinatorial optimization to evaluate observations.
This approach can become computationally intractable with increasing the robots' field of view, as it increases the number of candidate observations to assess. 
Additionally, learning-based methods have also been employed. In \cite{Li2019GraphNN}, the authors propose an architecture comprising a Convolutional Neural Network and a Graph Neural Network that compress and communicate information among robots for decentralized sequential path-planning.

\textbf{Contributions:}
In this paper, we assume a team of mobile robots that autonomously choose the optimal way to compress map data to assist another robot navigate an unknown environment. 
Our approach does not require prior environment knowledge, like learning-based methods, and is tractable in large maps, regardless of the robot configurations.

We propose a novel decoder-encoder pair to estimate the unknown environmental occupancy values and select the optimal compression, utilizing a given set of compressors (equivalently, quantizers).
We validate the effectiveness of our framework through simulations conducted on both a large, real-world-like map and a maze map.
While our simulations involve a single pair of robots, our framework can be readily extended to multiple robots.

\section{Preliminaries: Grid world and Abstractions} \label{sec:preliminaries}

We assume that the environment is represented by an occupancy grid in 2D (or 3D). 
A robot with onboard sensing capability can observe a 2D occupancy grid of size $w \times h$ as shown in Figure~\ref{fig:occupancygrid}(\subref{fig:occupancygrid1}). 
The occupancy values of the grid cells are stored in vector $x \in \Re^{wh}$. The $j$ component of $x$, denoted by $[x]_j$, is in the range $[0,1]$ for all $j = 1, \ldots, wh$.
Here $[x]_j$ denotes the traversability of the cell, where $[x]_j = 1$ indicates an untraversable cell and $[x]_j = 0$ denotes a free cell.
A compressed representation of the occupancy grid is described by an \textit{abstraction}, as in Figure~\ref{fig:occupancygrid}(\subref{fig:occupancygrid2}).

\begin{figure}[t]
    \centering    
    \begin{subfigure}{0.3 \linewidth}
        \centering
        \includegraphics[trim = 0 0 120 0, clip, width = \linewidth]{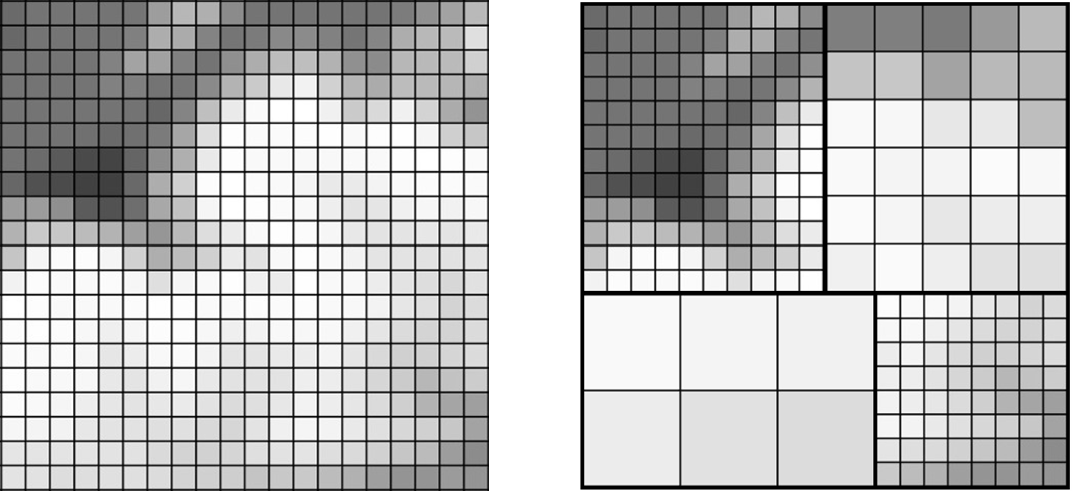}
        \put(-73,-9) {$\longleftarrow \hspace{.4 cm} w \hspace{.4 cm} \longrightarrow $}
        \put(-86, 1) {\rotatebox{90}{$\longleftarrow \hspace{.45 cm} h \hspace{.45 cm} \longrightarrow $}}
        \caption{}
        \label{fig:occupancygrid1}
    \end{subfigure} \hspace{ 1 cm}
    \begin{subfigure}{0.3 \linewidth}
        \includegraphics[trim = 120 0 0 0, clip, width = \linewidth]{OG2.png}
        \put(-73,-9) {\color{white}{$\longleftarrow \hspace{.4 cm} w \hspace{.4 cm} \longrightarrow $}}
        \caption{}
        \label{fig:occupancygrid2}
    \end{subfigure}
    \caption{(a) Full resolution occupancy grid; (b) Compressed (i.e., quantized) occupancy grid.}
    \label{fig:occupancygrid}
\end{figure}

Each abstraction is associated with a \textit{compression template} that generates a compressed representation of the occupancy grid.
This can be considered as a guide, indicating what and how the grid cells are going to be abstracted. 
The abstracted representation is a multi-resolution occupancy grid, where the occupancy of a compressed cell is determined by the occupancy values of the finest resolution cells that make up the compressed cell.
The occupancy value of a compressed cell can be determined using existing techniques, such as wavelets \cite{cowlagi2012multiresolution}, $k$-class trees \cite{kraetzschmar2004probabilistic} or information bottleneck methods \cite{larsson2020q}. 
For simplicity, in this study, we adopt the approach of computing the occupancy value of the compressed cell as the average of the underlying occupancy values of the finest resolution cells. This approach aligns with the principles of the information bottleneck method \cite{larsson2020q}.
Therefore, each abstraction can be thought of as a linear mapping ${\A}^{\theta}: [0,1]^{wh} \to [0,1]^{k_{\theta}}$, where $k_{\theta}\le wh$ is the number of cells in the compressed occupancy grid employing abstraction \({\theta}\). 
That is, for a given full resolution occupancy map $x \in [0,1]^{wh}$, the occupancy values of the cells in abstraction \({\theta}\) will be $o = \A^{\theta} x$, where $\A^{\theta} \in \Re^{k_{\theta} \times wh}$. 
Since we assume the occupancy value of a compressed cell to be the average of the occupancy values of the underlying finest resolution cells, the matrix $\A^{\theta}$ is row stochastic for all $\theta \in \Theta$, where $\Theta = \{1, \ldots, K\}$ is the set of available abstractions.

\subsection{Communication of Abstracted Environments} \label{sec:comm_of_abs}
Let $\Theta$ be known to both the Sender and Receiver robots. 
At every timestep $t$, the Sender selects an appropriate abstraction $\A^{\theta}$ to compress its observed occupancy grid and transmits it to the Receiver. 
Specifically, the Sender transmits the pair $(o_t, {\theta})$, where $o_t = \A^{\theta} x_t$, with $x_t$ representing the occupancy grid sensed by the Sender at time t. 
The receiver knows $\A^{\theta}$ and attempts to reconstruct $x_t$ from $o_t$. 

Let $n_m$ and $n_{i}$ denote the number of bits required to transmit an occupancy value ($[o]_j$) and an abstraction index (${\theta}$), respectively. 
Therefore, for abstraction \({\theta}\), the total number of bits $n_{\theta}$ required to transmit the abstracted grid is given by:
\begin{equation}\label{eq:bit_abstr}
n_{\theta} = k_{\theta} n_m + n_i.
\end{equation}
If the Sender were to send the full-resolution occupancy grid at each timestep, the required bits would be equal to $wh n_m$.

\section{Problem Formulation}

Consider a pair of mobile robots, a Seeker and a Supporter, that navigates through an unfamiliar environment  \(M \subset \Re^2\) repleted with static obstacles\footnote{
The framework extends in a straightforward manner to 3D environments.
}.
Let \(\textbf{p}_{A,t}\), \(\textbf{p}_{B,t} \in M\)  be the Seeker and Supporter's positions respectively at timestep \(t\), and let \(u_{t} \in U\) denote the Seeker's control action at time \(t\) selected from a finite set of control actions \(U\). 
The robots are equipped with sensors capable of observing a portion of the environment (local map) as they traverse it.
The Seeker's objective is to reach a designated target in minimum time by following a path generated by an online path-planning algorithm.
In contrast, the Supporter follows a predefined path and aims to assist the Seeker in achieving its objective by transmitting informative abstractions of its local map to the Seeker at each timestep. 
In this work, we do not adhere to a specific way to design the Supporter's path, but we consider it to be given, and prove the effectiveness of our algorithm regardless of it.
The Supporter's role can be likened to that of a drone assigned to reach a separate target or even a satellite in orbit, passing over the environment of the Seeker. 
Its path is determined a priori by a different objective and cannot be altered.
The proposed framework is shown in Figure~\ref{fig:flowchart}, and the roles of its components are delineated in Section \ref{sec:framework_arch}.

\begin{figure}[tb]
    \centering        
    {\includegraphics[width=0.85\linewidth]{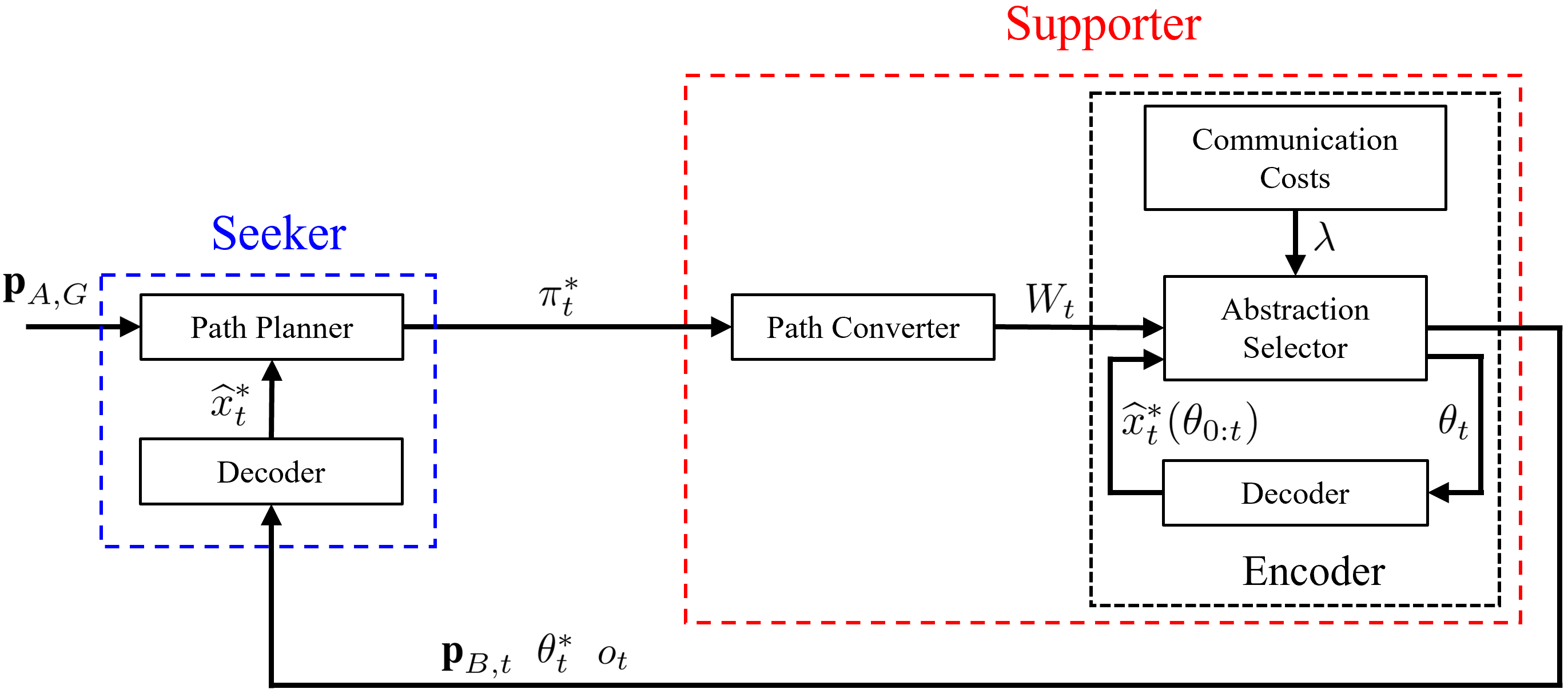}}
    \caption{Flowchart that presents the proposed framework's architecture at timestep \(t\). The Seeker shares its current optimal path \(\pi^*_t\) to the Supporter. The Supporter utilizes this to select the optimal abstraction of its local map \(\theta^*_t \in \Theta\) and sends it to the Seeker to aid in reaching its target.}
    \label{fig:flowchart}
\end{figure}

\subsection{Problem Statement}
Considering as inputs the initial positions of the Seeker and the Supporter, \(\textbf{p}_{A,0}\) and \(\textbf{p}_{B,0}\), along with the Supporter's predefined path in a global reference frame, we design a framework to online select the optimal abstraction \(\theta^*\), from a given set of abstractions \(\Theta\), to compress the Supporter's local map. 
The Supporter's encoder selects the abstractions, driven by the Seeker's transmitted path at every timestep. 
Meanwhile, the Seeker utilizes the accumulated measurements ($o_{0:t}$) to compute the control action \(u_t\) to reach its ultimate destination in minimum time (shortest path).

\section{Framework Architecture} \label{sec:framework_arch}
\subsection{Path Planner} \label{sec:path_planner}
Let \(G = (V,E)\) represent the graph associated with the occupancy grid environment \(M\), where \(V\) denotes the set of vertices and \(E\) the set of edges. 
Each vertex in \(V\) corresponds to a specific cell in \(M\) (with a slight abuse of notation, we will use \(\textbf{p}\) to denote both cell positions and graph vertices). 
Two vertices are deemed connected if the Seeker can move between them utilizing a control action \(u \in U\).
We assume that the time to traverse a cell is proportional to its occupancy value (difficulty to traverse) plus a constant (penalty for movement).
The cost of traversing a vertex is therefore given by \cite{larsson2021}:
\begin{equation}\label{eq:cell_cost}
    c_{\epsilon}  (\textbf{p})=     
    \begin{cases}
      x_{\textbf{p}} + a, & \text{if $\textbf{p} \in P_{\epsilon}$},\\
      N(\epsilon + a), & \text{otherwise},
    \end{cases}
\end{equation}
where \(x_{\textbf{p}} \in [0,1]\) is the occupancy value of the cell at position \(\textbf{p}\), \(a\) is a constant cost for traversing a cell, \(N\) is the total number of vertices or cells in $M$, and \(P_{\epsilon} = \{\textbf{p} \in V: x_{\textbf{p}} \leq \epsilon\}\) designates the set of cells meeting a feasibility condition, where \(\epsilon \in [0,1]\) is a scalar that defines cell feasibility.

Let \(\Pi\) denote the set of paths with the first element being the Seeker's current position \(\textbf{p}_{A,t}\) and the last element being its goal location \(\textbf{p}_{A,G}\). 
Then, the optimal path is given by:
\begin{equation}\label{eq:path}
    \pi^* = \argmin_{\pi \in \Pi}\sum_{\textbf{p} \in \pi}{c_{\epsilon} (\textbf{p})}.
\end{equation}

When \(\pi \subseteq P_{\epsilon}\), the path is referred to as an \(\epsilon\)-feasible path. 
By setting the second scale of \eqref{eq:cell_cost} larger than the cost of any feasible path, we exclude infeasible vertices, unless no feasible path is available. 
This ensures that the path-planning algorithm will always find a path.

Figure \ref{fig:example}(\subref{fig:ex_environ}) presents a simple example of an occupancy grid environment with obstacles, while Figure \ref{fig:example}(\subref{fig:ex_graph}) illustrates the associated graph. The graph is constructed by considering the set of control actions \(U = \)\{UP, DOWN, LEFT, RIGHT\}, and it showcases the optimal path (computed, for instance, using Dijkstra's graph search algorithm \cite{Dijkstra1959}, with cost function \eqref{eq:cell_cost}).

\begin{figure}[tb]
     \centering
     \begin{subfigure}[b]{0.29\linewidth}
         \centering
         \includegraphics[width=\textwidth]{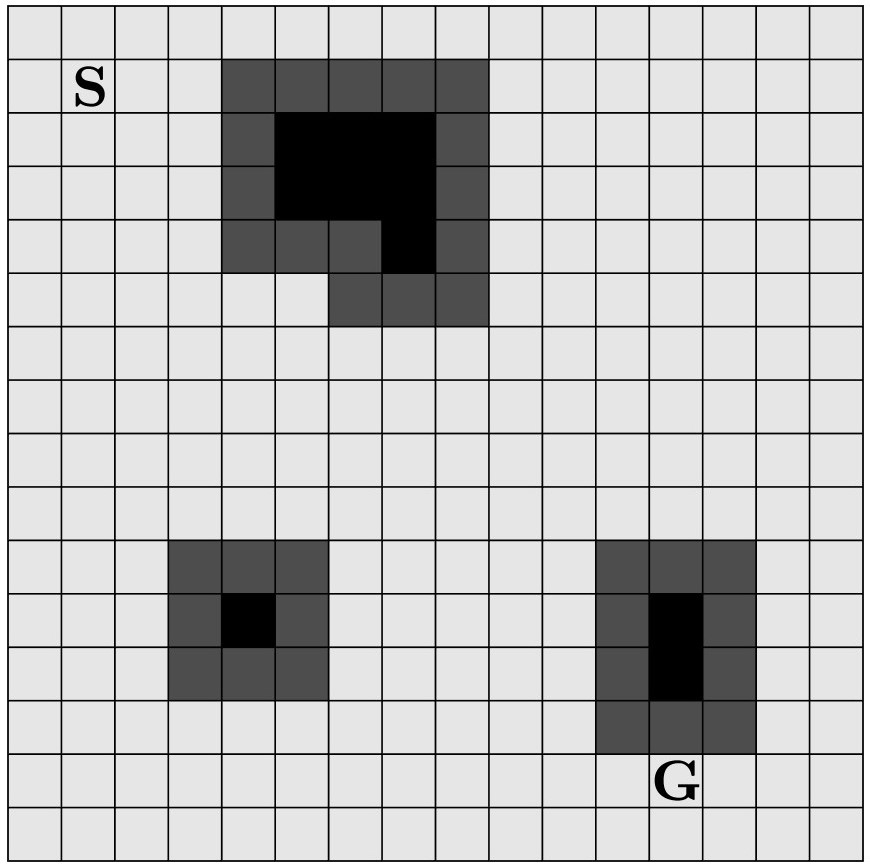}
         \caption{}
         \label{fig:ex_environ}
     \end{subfigure}
     \begin{subfigure}[b]{0.29\linewidth}
         \centering
         \includegraphics[width=\textwidth]{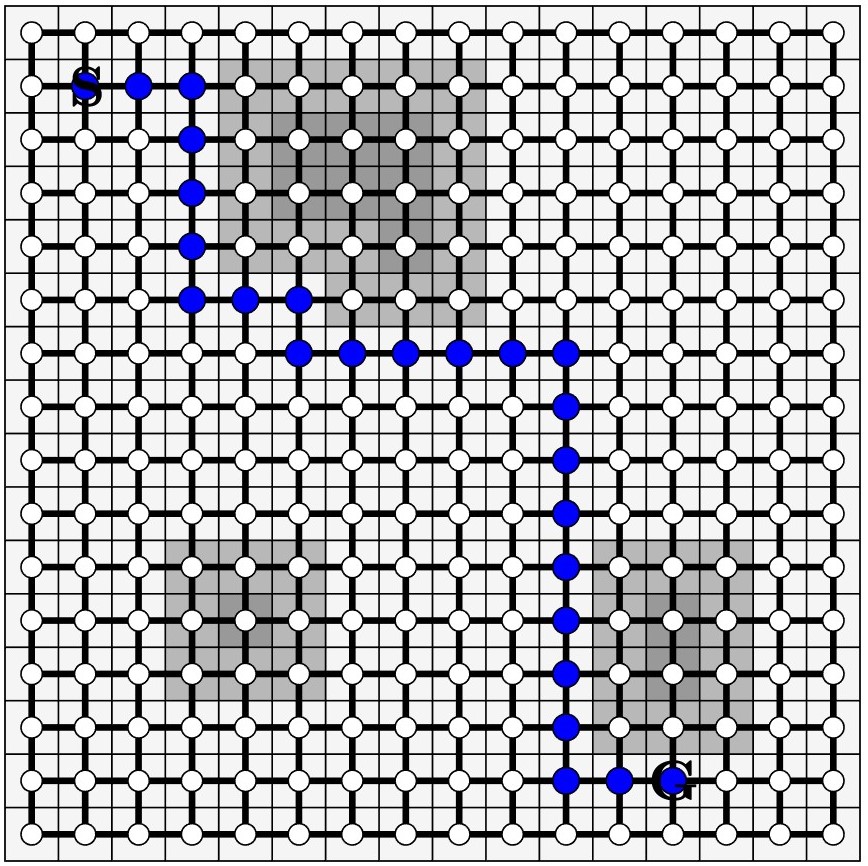}
         \caption{}
         \label{fig:ex_graph}
     \end{subfigure}
     \begin{subfigure}[b]{0.367\linewidth}
         \centering
         \includegraphics[width=\textwidth]{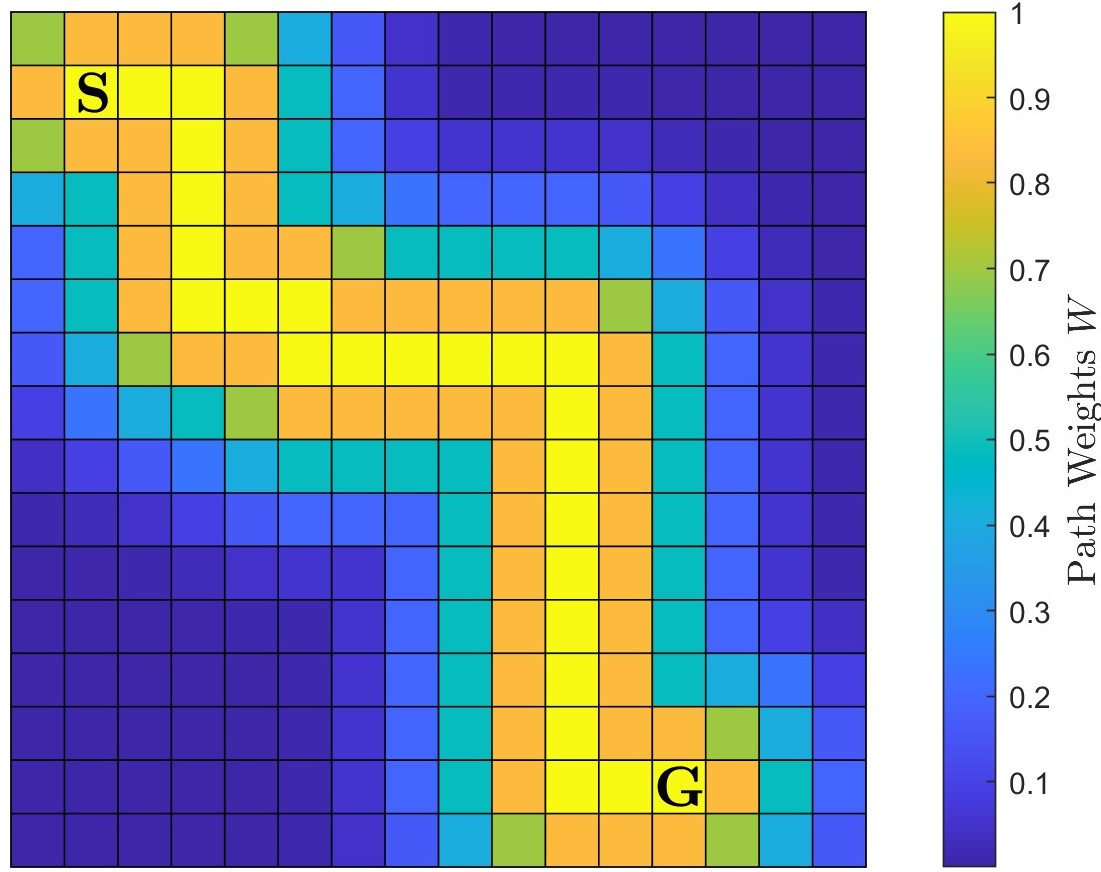}
         \caption{}
         \label{fig:ex_weights}
     \end{subfigure}
        \caption{(a) Example of a discretized environment. \(\textbf{S}\) denotes the starting cell while \(\textbf{G}\) denotes the target; (b) Associated graph of the discretized environment employing the set of actions \(U = \)\{UP, DOWN, LEFT, RIGHT\}, along with the optimal path (blue nodes) with cost function \eqref{eq:cell_cost}; (c) Path weights computed using \eqref{eq:path_weights_stand}.}
        \label{fig:example}
\end{figure}

\subsection{Decoder}
The decoder's primary role (see Figure \ref{fig:flowchart}) is to provide estimates for the vector \(x \in [0,1]^N\) containing the occupancy values of the cells of \(M\), where recall that \(N\) is the total number of cells in \(M\). 
To achieve this, it leverages both the past and present Seeker's measurements and Supporter's choices for abstractions. 

The Seeker's measurements as well as the Supporter's abstractions can be described by a set of linear equality and inequality constraints:
\begin{equation}\label{eq:constraints}
    C_t = \left\{
    x \in \Re^N:  
    \begin{matrix}
           \A_{0:t} x = o_{0:t} , \\
           0 \leq [x]_j \leq 1, \textrm{ } j = 1,\ldots,N
    \end{matrix}
    \right\}.
\end{equation}

Equation \eqref{eq:constraints} represents the intersection of a hyperplane with a hypercube in \(\Re^N\). 
Let \(k_t\) be the number of equality constraints in \eqref{eq:constraints} which is equal to the number of past and current Seeker's measurements and the past and current Supporter's abstractions, with redundant, linearly dependent equations removed. 
It becomes evident that at each timestep, the rows and elements of \(\A_{0:t} \in \Re^{k_t \times N}\) and \(o_{0:t} \in \Re^{k_t}\) might increase as new measurements and abstractions are added. 

In case the true values of certain elements of \(x\) are not perfectly known, it is imperative to establish a systematic method for computing estimates, utilizing the set \(C_t\). 
To provide such estimates, we assume that the occupancy vector \(x\) is a multivariate random variable following a distribution, known to both the Seeker and the Supporter.
Leveraging principles of stochastic estimation \cite{speyer2008}, we find the conditional expectation for each element of \(x\) within \(C_t\). 
This is achieved by identifying the point in \(C_t\) that minimizes the variance of the estimation error. 
Therefore, the vector \(\widehat{x}_t^{*}\) containing the estimates of \(x\) is given by: 
\begin{equation}\label{eq:decoder}
        \widehat{x}_t^{*} = \argmin_{  \widehat{x}_t \in C_t} \sum_{j = 0}^N \mathbb{E} [ [x]_j - [\widehat{x}_t]_j]^2.
    \end{equation}

Given a distribution for \(x\), \eqref{eq:decoder} can be transformed into a convex optimization problem with linear constraints.

\begin{proposition}
    Let $x$ follow the distribution $f(\cdot)$ with mean $\mu$ and covariance $\Sigma$. 
    Then, 
    \begin{align}\label{eq:decod_fin}
        \widehat{x}^*_t = \argmin_{  \widehat{x}_t \in C_t} \| \widehat{x}_t - \mu \|^2.
    \end{align}
\end{proposition}
\begin{proof}
Equation (\ref{eq:decoder}) can be equivalently written as:
    \begin{equation}\label{eq:decod_expect_val}
        \widehat{x}_t^{*} = \argmin_{  \widehat{x}_t \in C_t} \mathbb{E} [\| x - \widehat{x}_t\|^2].
    \end{equation}

The covariance matrix is defined as \(\Sigma = \mathbb{E} [xx^\top] - \mathbb{E} [x]\mathbb{E} [x]^\top\). 
Hence, the trace of \(\Sigma\) is given by:
    \begin{align*}
        \tr(\Sigma) =&  \tr(\mathbb{E} [xx^\top]) - \tr(\mathbb{E} [x]\mathbb{E} [x]^\top)\\
        =&  \mathbb{E} [\tr(x^\top x)] - \tr(\mathbb{E} [x]^\top\mathbb{E} [x])\\
        =& \mathbb{E}[\|x\|^2] - \|\mu\|^2,
    \end{align*}
where \(\tr(\cdot)\) denotes the trace of a matrix.

Notice that, 
    \begin{align*}
        \mathbb{E} [\| x - \widehat{x}_t\|^2] =& \mathbb{E}[\|x\|^2] + \|\widehat{x}_t\|^2 - 2 \mathbb{E}[x]^\top \widehat{x}_t \\
        =& \tr(\Sigma) + \|\mu\|^2 + \|\widehat{x}_t\|^2 - 2 \mu^\top \widehat{x}_t \\
        =& \tr(\Sigma)+ \| \widehat{x}_t - \mu \|^2.
    \end{align*}
    
    Thus, 
    \begin{align*}
        \widehat{x}^*_t = \argmin_{  \widehat{x}_t \in C_t}  \mathbb{E} [\| x - \widehat{x}_t\|^2]  = \argmin_{  \widehat{x}_t \in C_t} \| \widehat{x}_t - \mu \|^2.
    \end{align*}
\end{proof}

Given that $C_t$ is a convex set (polyhedron) and the objective function is quadratic, the optimization problem is convex for every distribution $f(\cdot).$ It is further noteworthy that the optimal solution $\widehat{x}^*_t$ depends only on the mean of the distribution $f(\cdot)$.
In this work, we set $\mathbb{E}[[x]_j] = 0.5$, to help the decoder in its estimation since it lacks prior information about the environment and $[x]_j$ is bounded between $0$ and $1$.
Therefore:
\begin{align}\label{eq:decod_paper}
     \widehat{x}^*_t = \argmin_{  \widehat{x}_t \in C_t} \| \widehat{x}_t -  \frac{1}{2} \mathbf{1}\|^2,
\end{align}
where $\textbf{1}$ is a vector of all ones.

\subsection{Path Converter}
The Seeker sends its current optimal path $\pi^*_t$ obtained from \eqref{eq:path} to the Supporter.
The Supporter utilizes this information to guide its abstraction selection process. 
This is achieved by assigning weights to each cell within \(M\) based on its proximity to the path. 
To compute these weights, we employ a normalized Gaussian function:
\begin{equation}\label{eq:path_weights_stand}
    w_t(\textbf{p}) = \max_{\textbf{p}_{\pi} \in \pi^*_t} e^{-\frac{\|\textbf{p}-\textbf{p}_{\pi}\|^2}{2\sigma^2}}, \qquad \textbf{p} \in M,
\end{equation}
where \(\sigma\) is a parameter characterizing the width of the curve around the path. 

Figure \ref{fig:example}(\subref{fig:ex_weights}) illustrates the path weights after applying \eqref{eq:path_weights_stand} to the example in Figure \ref{fig:example}(\subref{fig:ex_environ}).

\subsection{Encoder}

\begin{algorithm}[tb]
\caption{The Encoder's Algorithm}\label{alg:encoder}
\hspace*{\algorithmicindent} \textbf{Input: }{$W_t$, $\A_{0:t-1}$, $o_{0:t-1}$, $\textbf{p}_{A,0:t}$, $\textbf{p}_{B,0:t}$} \\
\hspace*{\algorithmicindent} \textbf{Output: } {$\theta^{*}_t$, $o_{t}$}
\begin{algorithmic}[1]
    \ForAll {$\textbf{p} \in (L_{B,t} \cap M_{A,t}) \cup (L_{A,t} \cap M_{B,t})$}
        \State $(\A,o) \leftarrow \textsc{Update}(\A_{0:t-1},o_{0:t-1})$
    \EndFor
    
    \ForAll{$\theta \in \Theta$}
        \State $(\A^{\theta}_t,o^{\theta}_t) \leftarrow $ \textsc{Abstractions}($L_{B,t}, \theta$)
        
        \State $(\A_{0:t},o_{0:t}) \leftarrow$ \textsc{Independent}
            $\left( \begin{bmatrix}
                \A & o \\
               \A^{\theta}_t & o^{\theta}_t
            \end{bmatrix} \right)$
        
        \State $\widehat{x}_t^{*} \leftarrow$ \textsc{Decoder}$(\A_{0:t},o_{0:t})$ using (\ref{eq:decod_paper})
        \State $J_t(\theta) \leftarrow$ \textsc{ComJ}($W_t$, $\widehat{x}_t^{*}$, $\widetilde{x}_t(\textbf{p}_{B,0:t})$, $\lambda$) using (\ref{eq:encoder_a})
    \EndFor
    \State $\theta^{*}_t \leftarrow  \argmin J_t$
    \State $o_t \leftarrow  o^{\theta^{*}_t}_t$
    \State \Return $\theta^{*}_t$, $o_t$

\end{algorithmic}
\end{algorithm}

\begin{figure*}[!t]
    \centering   
     \begin{subfigure}[b]{0.20\linewidth}
         \centering
         \includegraphics[width=\textwidth]{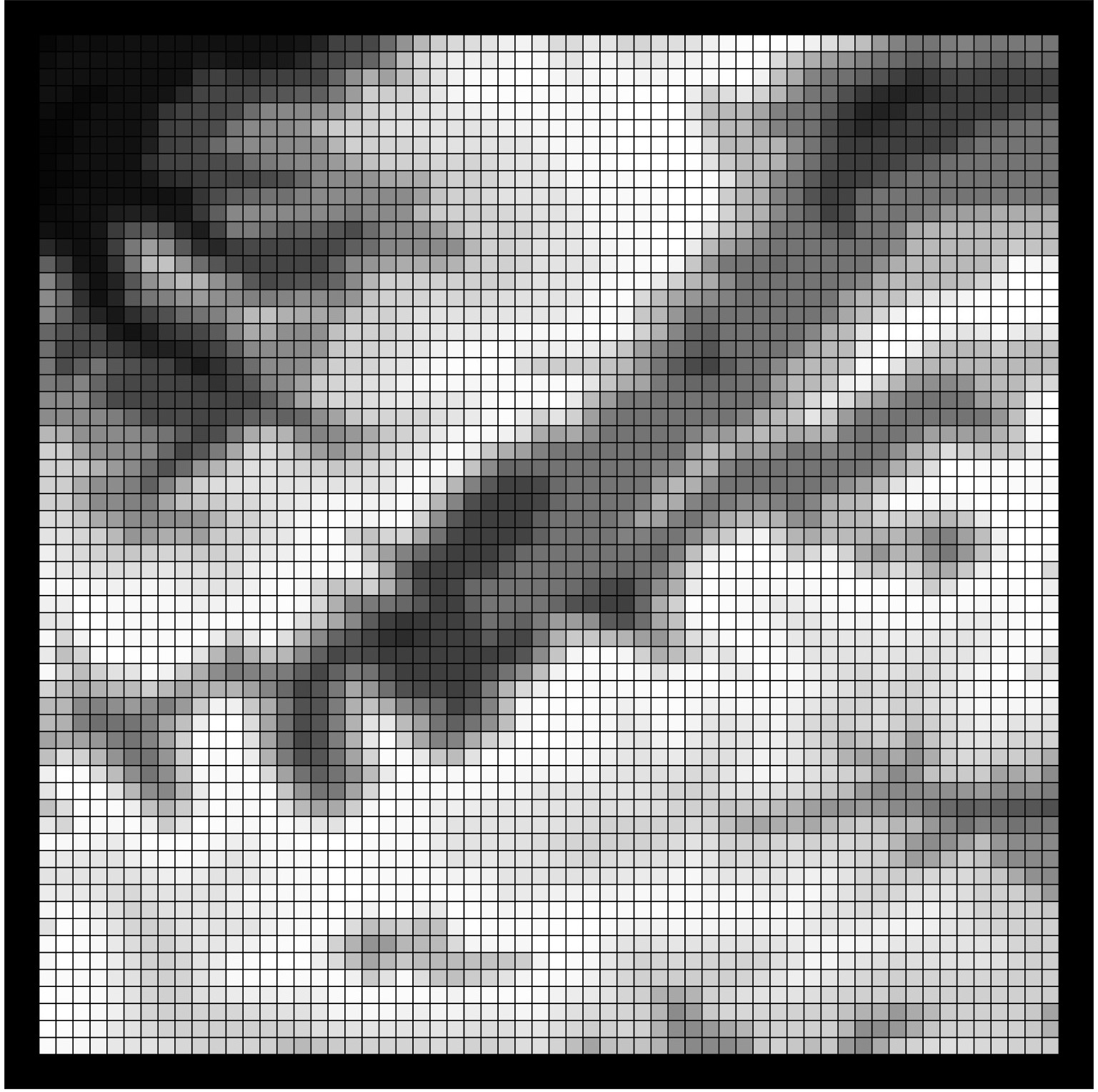}
         \caption{Realistic Map}
         \label{fig:real_map}
     \end{subfigure}
     \hspace{1mm}
     \begin{subfigure}[b]{0.20\linewidth}
         \centering
         \includegraphics[width=\textwidth]{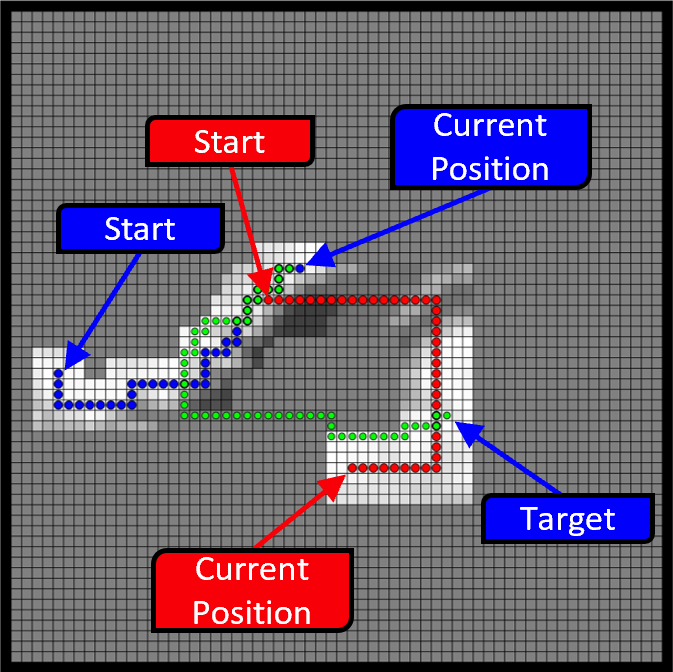}
         \caption{FI framework ($t=40$)}
         \label{fig:frame1_40}
     \end{subfigure}
     \hspace{1mm}
     \begin{subfigure}[b]{0.20\linewidth}
         \centering
         \includegraphics[width=\textwidth]{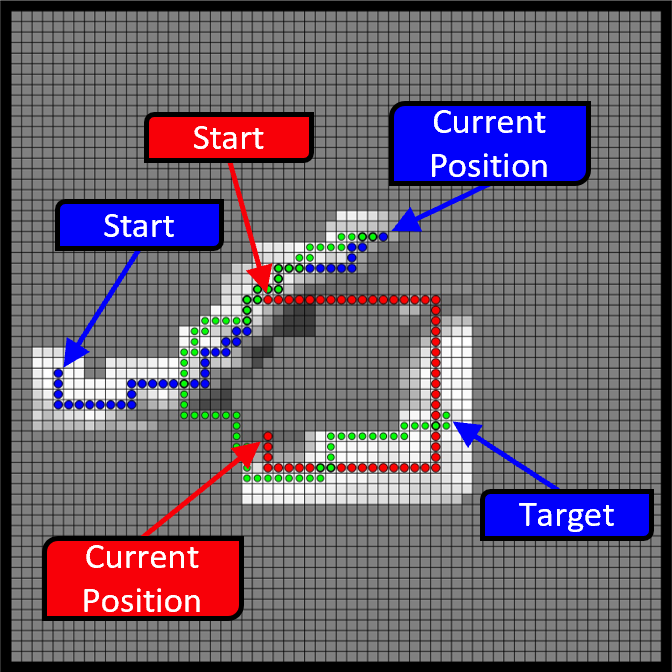}
         \caption{AS framework ($t=51$)}
         \label{fig:frame2_51}
     \end{subfigure}
    
     \begin{subfigure}[b]{0.20\linewidth}
         \centering
         \includegraphics[width=\textwidth]{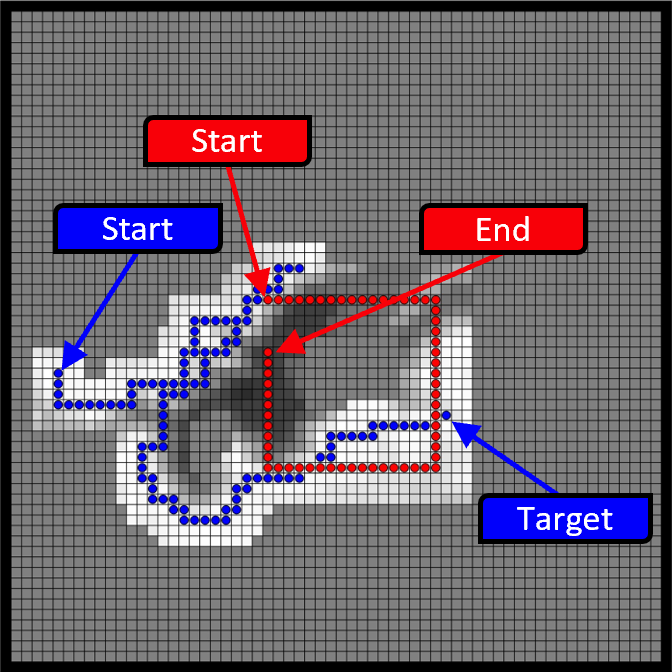}
         \caption{FI framework (final)}
         \label{fig:frame1_final}
     \end{subfigure}
    \hspace{1mm}
     \begin{subfigure}[b]{0.20\linewidth}
         \centering
         \includegraphics[width=\textwidth]{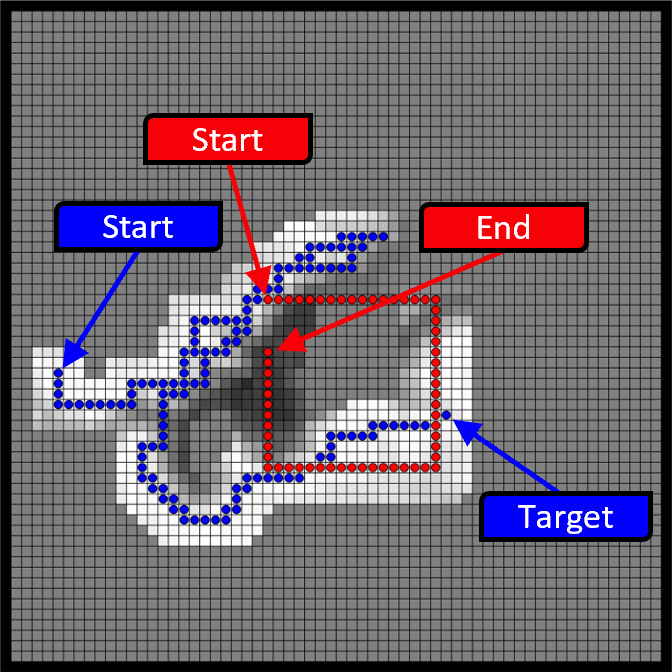}
         \caption{AS framework (final)}
         \label{fig:frame2_final}
     \end{subfigure}
    \hspace{1mm}
     \begin{subfigure}[b]{0.20\linewidth}
         \centering
         \includegraphics[width=\textwidth]{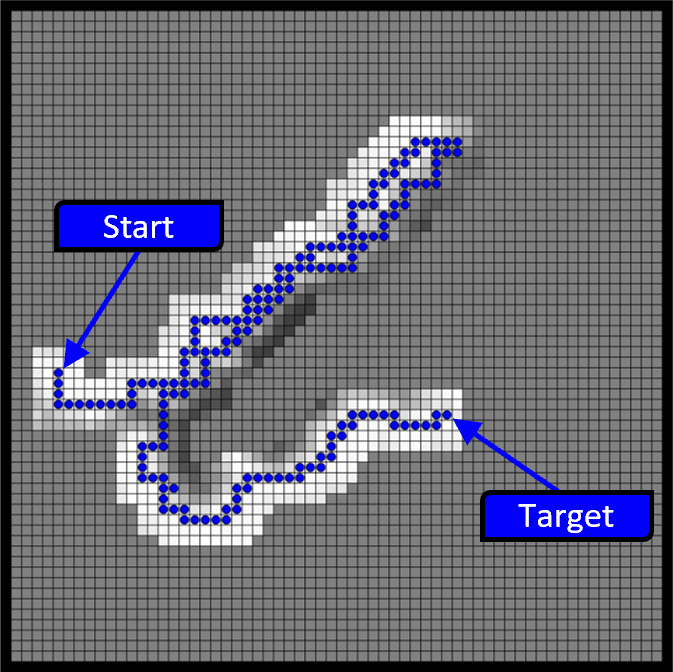}
         \caption{U framework (final)}
         \label{fig:frame3_final}
     \end{subfigure}
    \caption{Simulation example using the realistic map. Figures (b)-(f) illustrate the Seeker's estimated map at different instances. The Supporter's path is illustrated with red and its initial position is \(\textbf{p}_{B,0} = (26,36)\). The cells that the Seeker has already traversed are presented with blue, while green is the current path constructed by its path-planning algorithm.}
    \label{fig:real_map_sim}
\end{figure*}

The encoder's role is to select the optimal abstraction from a given set, to transmit to the Seeker. 
This selection is conducted with a focus on both navigation and communication aspects, considering the path weights \(w_t(\textbf{p})\) and penalizing abstractions based on their required transmission bandwidth.

Let \(\theta \in \Theta\) denote an abstraction, where \(\Theta\) is defined in Section~\ref{sec:comm_of_abs}. 
The optimal abstraction at each timestep \(t\) is derived through the minimization of the following criterion:
\begin{subequations}\label{eq:encoder}
    \begin{eqnarray}
         & J_t(\theta_{0:t}) = \| W_t \circ (\widetilde{x}_t-\widehat{x}_t^{*}(\theta_{0:t}))\|^2 + \lambda(\theta_t), \label{eq:encoder_a}\\
         & \theta^*_t = \argmin_{\theta_t \in \Theta} J_t(\theta_{0:t}), \label{eq:encoder_b}
    \end{eqnarray}
\end{subequations}
where \(W_t\) is the vector containing the weights \(w_t(\textbf{p})\) of every cell, \(\widetilde{x}_t \in \Re^N\) is the total sensed occupancy grid by the Supporter until timestep \(t\),  \(\widehat{x}_t(\theta_{0:t}) \in \Re^N\) is the estimation vector (depending on the history of abstractions), \(\circ\) is the Hadamard product, and \(\lambda(\theta_t)\) is the communication/bandwidth cost assigned to each abstraction. 
Similar to previous works in classical control \cite{maity_quant}, we do not impose any specific structure on \(\lambda(\theta_t)\), while considering that the cost is proportional to the number of compressed cells in an abstraction. 
This approach enables us to capture the idea of penalizing abstractions based on their resolution.

It is important that the Supporter's encoder selects the optimal abstraction for the specific Seeker's decoder (see Figure \ref{fig:flowchart}). 
Therefore, the estimation vector \(\widehat{x}_t(\theta_{0:t})\) is computed using (\ref{eq:decoder}) with the constraint set given by:
\begin{equation}\label{eq:constraints_encoder}
    C^{\theta}_t = \left\{
    x \in \Re^N:  
    \begin{matrix}
       \begin{bmatrix}
    \A \\
    \A^{\theta}_t
\end{bmatrix} 
         x =                \begin{bmatrix}
    o \\
    o^{\theta}_t
\end{bmatrix} 
        \\
        0 \leq [x]_{j} \leq 1, \textrm{ } j = 1,\ldots,N
    \end{matrix}
    \right\},
\end{equation}
where $\A^{\theta}_t$ and $o^{\theta}_t$ denote the matrix of the candidate abstraction \(\theta\) and the occupancy values, respectively, at timestep $t$; see Section~\ref{sec:preliminaries}, and \((\A,o)\) is described in the next paragraph.

The encoder's algorithm is given in Algorithm \ref{alg:encoder}. Here \(L_{A,t} = L_{A,t}(\textbf{p}_{A,t}) 	\subseteq M\) and \(L_{B,t} = L_{B,t}(\textbf{p}_{B,t}) \subseteq M\) are sets that contain the cells of the Seeker and Supporter's current local map respectively, and \(M_{A,t} = M_{A,t}(\textbf{p}_{A,0:t}) \subseteq M\) and \(M_{B,t}(\textbf{p}_{B,0:t}) \subseteq M\) are sets that contain all the sensed/finest resolution cells from timestep \(0\) to \(t\) of the Seeker and the Supporter respectively.
Lines 1-3 in Algorithm \ref{alg:encoder} incorporate the values of the finest resolution cells of the Supporter's local map that have already been measured by the Seeker directly to the set \(C^{\theta}_t\). 
These cells are excluded from the abstraction process. 
Moreover, they inform the Supporter if the Seeker has just measured cells previously included in transmitted abstractions, and add them to the set \(C^{\theta}_t\).
Thus, the function \(\textsc{Update}\) adds the elements corresponding to these measurements to \(\A_{0:t-1}\) and \(o_{0:t-1}\).
Line 5 uses the function \(\textsc{Abstractions}\) to compute the pair \((\A^{\theta}_t,o^{\theta}_t)\) by applying abstraction \(\theta\) to \(L_{B,t}\), while Line 6 uses the function  \(\textsc{Independent}\) to concatenate all the measurements and uses a method (i.e., Gaussian elimination) to exclude linear dependent equations. 
Lines 4-9 perform an exhaustive search to find the optimal abstraction, as the available set of abstractions is relatively limited. 
In future work, we plan to leverage the dependence of different abstractions, to increase the set of abstractions, and apply a more sophisticated search.

\section{Experiments}

In this section, we present the simulation results to validate the effectiveness of our framework. 
We conducted 500 simulations on each of two different 2D maps: a realistic map (\(64 \times 64\)) with probabilistic occupancy values (Figure \ref{fig:real_map_sim}(\subref{fig:real_map})), and a maze (\(30 \times 30\)) with deterministic values (Figure \ref{fig:maze_map}).
A single pair of a Seeker and a Supporter is employed for both maps, with both robots initiating movement simultaneously and traversing one cell per timestep. 
The Supporter has the ability to move over obstacles (i.e., it is an aerial vehicle, e.g., a surveillance drone).
The Seeker's local map \(L_{A,t}\) size is \(5 \times 5\) cells in the first scenario and \(3 \times 3\) cells in the second. 
Furthermore, the Supporter has a field of view \(L_{B,t}\) of \(7 \times 7\) cells in both scenarios. 
Both robots are positioned at the center of their respective local maps.

In the first scenario, we tested different Supporter's predefined paths while maintaining the same initial and final positions of the Seeker. 
In the second scenario, we vary the initial and final positions of the Seeker while retaining the same predefined path.
The Seeker's Path Planner uses the cell cost given in (\ref{eq:cell_cost}) with values of the constants \(a = 0.025\) and \(\epsilon = 0.501\).
The Supporter's encoder utilizes a finite set of 10 abstractions as shown in Figure \ref{fig:abstr_set} and with values of the parameters in \eqref{eq:path_weights_stand} \(\sigma = 10\) and \(\sigma = 3.33\) for the real-world-like and maze scenarios, respectively.

\subsection{Performance Metrics}
We compare our (Abstraction Selection - AS) framework  with two alternatives: a Fully-Informed (FI) framework and an Uninformed (U) framework.
In the FI framework, the Supporter transmits all the new measurements contained in \(L_{B,t}\) at each timestep \(t\).
In the U framework, the Seeker reaches its destination without assistance from the Supporter.

As explained in Section \ref{sec:path_planner}, we assumed that the time required to traverse a cell is proportional to the cell's cost. 
Hence, the total time taken by the Seeker to reach its target is proportional to the accumulated cost \(\mathcal{C}\).
Let \(\pi_f\) include the cells that the Seeker traversed until it reached its target, without excluding duplicate cells.
Hence, \(\mathcal{C}\) is given by:
\begin{equation}\label{eq:acc_cost}
    \mathcal{C} = \sum_{\textbf{p} \in \pi_f}{c_{\epsilon} (\textbf{p})}.
\end{equation}
We assess our framework's effectiveness in assisting the Seeker to reach its destination, by calculating the total number of simulations in which the Seeker had the highest \(\mathcal{C}\), in comparison to the other two frameworks, and we classify these simulations as \textquotedblleft failure\textquotedblright.
Meanwhile, we classify simulations that resulted in the same \(\mathcal{C}\) as \textquotedblleft neutral\textquotedblright.

We also compute the average time ratio: 
\begin{equation}\label{eq:bit_ratio}
    r_{\textrm{time},i} = \frac{1}{n_{\textrm{sim}}}\sum_{s = 0}^{n_{\textrm{sim}}}\frac{\mathcal{C}_{i}(s)}{\mathcal{C}_o(s)},
\end{equation}
where \(n_\textrm{sim}\) is the total simulation number, \(\mathcal{C}_o(s)\) is the accumulated cost by using the optimal framework at simulation \(s\), \(i\) is the framework index (i.e., \(i=1\) for FI, \(i=2\) for AS, and \(i=3\) for U framework), and \(\mathcal{C}_i(s)\) is the accumulated cost of framework \(i\) at simulation \(s\).

Additionally, we evaluated the performance of our framework in reducing the amount of information sent at each timestep, by calculating the average ratio of bits sent by our framework and the bits sent by the FI framework:
\begin{equation}\label{eq:bit_ratio2}
    r_{\textrm{bits}} = \frac{1}{n_{\textrm{sim}}} \sum_{s = 0}^{n_{\textrm{sim}}} \frac{\sum_{t = 0}^{T_{B,\theta}(s)}{n_{\theta,t}}(s)}{\sum_{t = 0}^{T_{B,f}(s)}{n_{f,t}}(s)},
\end{equation}
where \(n_{\theta,t}(s)\) is given in (\ref{eq:bit_abstr}) and denotes the bits sent by our framework's Supporter using abstraction \(\theta\) at timestep \(t\) and simulation \(s\), \(n_{f,t}(s)\) are the bits sent by the FI framework at timestep \(t\) and simulation \(s\), and \(T_{B,\theta}(s)\) and \(T_{B,f}(s)\) is the time horizon that the Supporter transmits information at simulation \(s\) for the AS and FI framework respectively. 
The parameter values in \eqref{eq:bit_abstr} are \(n_m = 12\), \(n_i = 4\).

\begin{figure}[tb]
    \centering        
    {\includegraphics[width=0.62\linewidth]{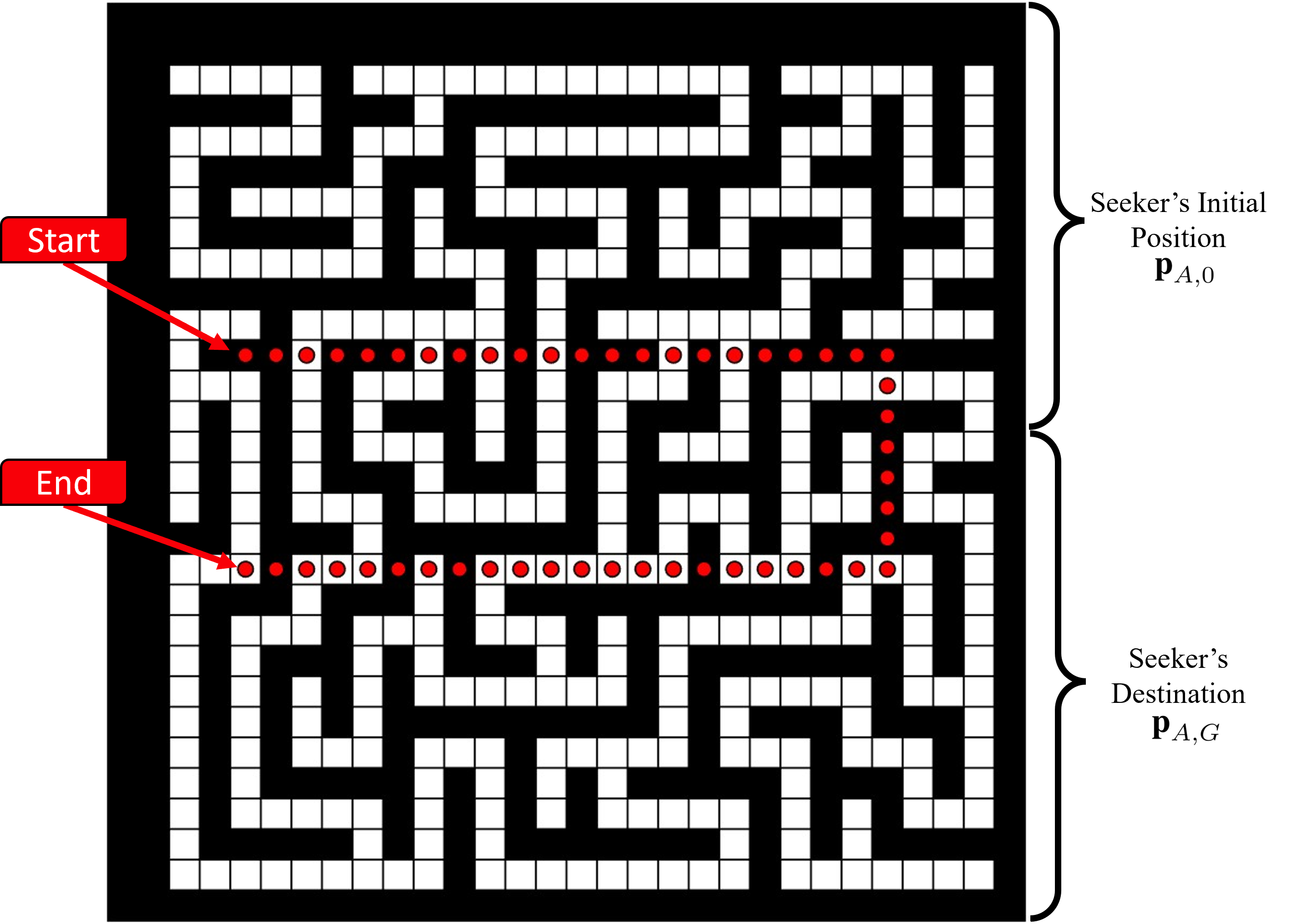}}
    \caption{Maze Map. The Supporter's path is defined with red.}
    \label{fig:maze_map}
\end{figure}

\begin{figure}[tb]
    \centering        
    {\includegraphics[width=0.6\linewidth]{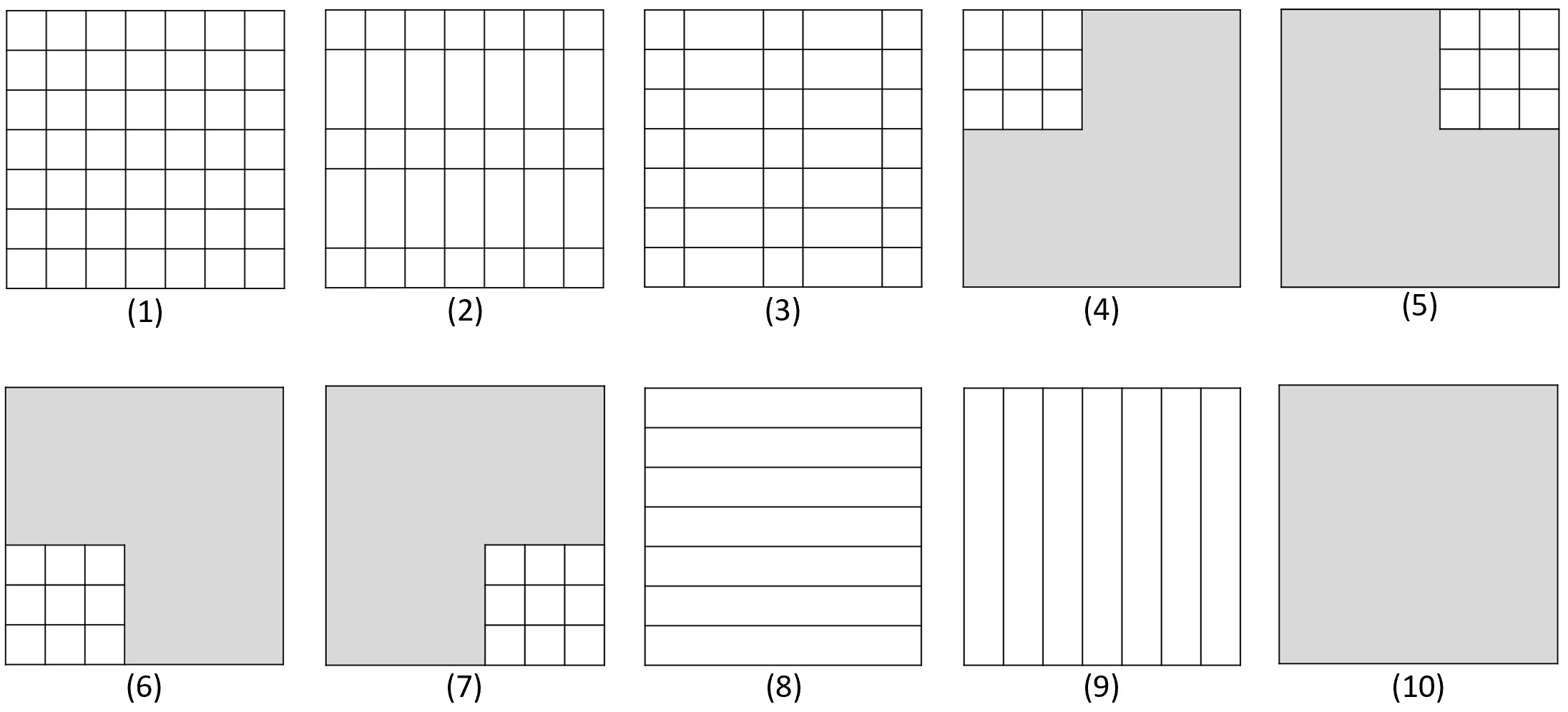}}
    \caption{Set of Abstractions used for the simulations. The grey area denotes no information about these cells}
    \label{fig:abstr_set}
\end{figure}

\subsection{Variation of Supporter's Path}
In the realistic map environment (Figure \ref{fig:real_map_sim}(\subref{fig:real_map})), we conducted 500 simulations with different predefined paths of the Supporter. 
The Seeker's initial position is \(\textbf{p}_{A,0} = (6,29)\) and the destination is \( \textbf{p}_{A,G} = (43,25)\). 
The Supporter's initial position on the map is arbitrary and determines which of the four different paths it will follow with \(T_{B,\max} = 60\).

Table \ref{tab: real_map} presents the results. 
We conclude that the FI framework produced the highest \(\mathcal{C}\) for the smallest number of simulations. 
However, we also observe that there were 28 simulations where the FI framework had the highest \(\mathcal{C}\). 
This means that information might not always be beneficial for the Seeker but, on the contrary, it might be misleading. 
Informing the Seeker about obstacle-free areas that lead to potential dead-ends may cause it to enter these areas. Conversely, alerting the Seeker to blocked areas may lead it to mistakenly avoid regions with clear paths nearby.

In conclusion, our framework, on average, increased the time by 26.9\% whereas the FI framework increased it by 2.2\%, and the U framework by 141.0\%. 
However, our framework also achieved a 62.7\% reduction in transmitted information, while, on average, maintained a satisfactory performance, comparing to the other two alternatives.

Figure \ref{fig:real_map_sim} illustrates one of the 500 conducted simulations on the real-world-like environment. In this example, we observe the effectiveness of our framework (Figure \ref{fig:real_map_sim}(\subref{fig:frame2_51})) in providing the Seeker with information that prompts it to change direction and follow the correct path faster than the U framework, resulting in a shorter time. Furthermore, the FI framework (Figure \ref{fig:real_map_sim}(\subref{fig:frame1_40})) informs the Seeker to change its path more rapidly than both of the other frameworks.

\begin{table}[!t]
\caption{Realistic Map Results}
\label{tab: real_map}
\centering
\scriptsize
\begin{tabular}{ |c|c|c|c| } 
    \hline
    \textbf{Framework} & Fully-Informed & Abstraction Selector & Uninformed \\ 
    \hline
    \textbf{failures} & 28 & 84 & 365 \\ 
    \hline
    \textbf{neutral} & 11 & 11 & 11 \\ 
    \hline
    \(r_\textrm{time}\) & 1.022 & 1.269 & 2.410 \\ 
    \hline
    \(r_\textrm{bits}\) & - & 0.373 & - \\ 
    \hline
\end{tabular}
\end{table}

\subsection{Variation of Seeker's Initial and Final Positions}
In the maze map (Figure \ref{fig:maze_map}), we run 500 simulations for different initial position of the Seeker \(\textbf{p}_{A,0}\) and the target \(\textbf{p}_{A,G}\), while keeping the Supporter's path the same. 
The initial position of the Supporter is \(\textbf{p}_{B,0} = (5,19)\) and \(T_{B,\max} = 49\).

Table \ref{tab: maze_map} presents the simulation results. Our framework, on average, increased time by 14.8\% whereas the FI framework increased it by 5.5\%, and the U framework by 34.3\%.
Nonetheless, our framework also managed to reduce transmitted information by 43.6\%, while increasing the time only by 9.3\%, compared to the optimal FI framework.

\begin{table}[!t]
\caption{Maze Results}
\label{tab: maze_map}
\centering
\scriptsize
\begin{tabular}{ |c|c|c|c| } 
    \hline
    \textbf{Framework} & Fully-Informed & Abstraction Selector & Uninformed \\ 
    \hline
    \textbf{failures} & 28 & 59 & 306 \\ 
    \hline
    \textbf{neutral} & 77 & 77 & 77 \\ 
    \hline
    \(r_\textrm{time}\) & 1.055 & 1.148 & 1.343 \\ 
    \hline
    \(r_\textrm{bits}\) & - & 0.564 & - \\ 
    \hline
\end{tabular}
\end{table}

\section{Conclusions}
This paper addresses the challenge of determining the optimal information compression for communication in the context of mobile robot path-planning.
We assume a team of mobile robots that compress their local maps to assist another robot reach a destination in an unfamiliar environment.
In contrast with existing methods, our framework does not require prior knowledge of the environment and is effective for various robot configurations and map sizes.
Simulation results validate the effectiveness of our framework. 
On average, our framework reduced the amount of information by approximately 50\% while maintaining satisfactory performance.
In the future, we plan to extend our framework to the multi-robot path-planning problem. 
We also intend to design a more sophisticated search method, utilizing abstraction dependence, to increase the set of abstractions.

\bibliographystyle{IEEEtran}
\bibliography{IEEEabrv,refs,Maity}

\end{document}